\documentclass[10pt,logo,copyright]{nvidiatechreport}
\usepackage{natbib}
\usepackage[utf8]{inputenc} % allow utf-8 input
\usepackage[T1]{fontenc}    % use 8-bit T1 fonts
\usepackage{hyperref}       % hyperlinks
\usepackage{url}            % simple URL typesetting
\usepackage{booktabs}       % professional-quality tables
\usepackage{amsfonts}       % blackboard math symbols
\usepackage{xcolor}
\usepackage{xspace}
\usepackage{nicefrac}       % compact symbols for 1/2, etc.
\usepackage{amsthm}
\newtheorem{definition}{Definition}
\usepackage{microtype}      % microtypography
\usepackage{xcolor}         % colors
\usepackage[utf8]{inputenc} % allow utf-8 input
\usepackage[T1]{fontenc}    % use 8-bit T1 fonts
\usepackage{url}            % simple URL typesetting
\usepackage{booktabs}       % professional-quality tables
\usepackage{amsfonts}       % blackboard math symbols
\usepackage{nicefrac}       % compact symbols for 1/2, etc.
\usepackage{microtype}      % microtypography
\usepackage{xcolor}         % colors
\usepackage[utf8]{inputenc} 
\usepackage[T1]{fontenc}    
\usepackage{hyperref}
\hypersetup{colorlinks}
\usepackage{url}           
\usepackage{booktabs}       
\usepackage{amsmath}       
\usepackage{nicefrac}     
\usepackage{microtype}      
\usepackage{multirow}      
\usepackage{graphicx}
\usepackage{graphics}
\usepackage{tabularx}
\usepackage{makecell}
\usepackage{caption}
\usepackage{wrapfig}
\usepackage{subcaption}
\usepackage{amsfonts}      
\usepackage{nicefrac}
\usepackage{amsmath,bm}
\usepackage{enumitem}
\usepackage[table]{xcolor}
\usepackage{graphicx}
\usepackage{amsmath,amssymb}
\usepackage{geometry}
\usepackage{amsthm}

\usepackage{titletoc}
\usepackage{minitoc}        % for \startcontents and \printcontents
\usepackage{tocloft} % For better ToC customization

\usepackage{algorithm}
\usepackage{verbatim}
\usepackage{listings}
\usepackage{wrapfig}
\usepackage{mathtools}
\definecolor{codeblue}{rgb}{0.25,0.5,0.25}
\definecolor{codekw}{rgb}{0.85, 0.18, 0.50}
\usepackage{subcaption}
\usepackage{tikz}
\usepackage{comment}
\usepackage{overpic}
\usepackage{algorithm}
\usepackage{algpseudocode}

\theoremstyle{plain}
\newtheorem{theorem}{Theorem}[section]
\newtheorem{proposition}[theorem]{Proposition}

\theoremstyle{definition}

\theoremstyle{remark}

\usepackage{booktabs}
\usepackage{xcolor,colortbl}
\usepackage{pgfplots}\pgfplotsset{compat=1.18}
\usepackage{pgfplots}
\pgfplotsset{compat=1.18}

\renewcommand{\arraystretch}{1.12}

\usepackage{amsmath,amsfonts,bm}

\def\eqref#1{equation~\ref{#1}}
\def\1{\bm{1}}

\DeclareMathAlphabet{\mathsfit}{\encodingdefault}{\sfdefault}{m}{sl}
\SetMathAlphabet{\mathsfit}{bold}{\encodingdefault}{\sfdefault}{bx}{n}

\usepackage{hyperref}
\usepackage{url}

\usepackage[most]{tcolorbox} % loads skins, breakable, theorems, etc.

\usepackage{pgfplots}
\usepackage{pgfplotstable}
\pgfplotsset{compat=1.18}
\usetikzlibrary{patterns} % Required for the hatched lines
\usepackage{tcolorbox}

\usepackage{tcolorbox}
\definecolor{myorange}{RGB}{217, 119, 37} % Similar to Qwen/DeepSeek bars
\definecolor{myteal}{RGB}{60, 174, 163}   % Similar to OpenReasoning bars

\newtcolorbox{takeaway}[3][]{
  enhanced,
  sharp corners,
  boxrule=0pt,
  colback=#2!4!white,
  borderline west={3.5pt}{0pt}{#2!85!black},
  drop fuzzy shadow=black!30!white,
  arc=2.6mm,
  left=6mm, right=6mm, top=4mm, bottom=4mm,
  coltitle=black,
  title={\small\bfseries #3}, % slimmer title font
  attach boxed title to top left={yshift=-3.0mm, xshift=4.0mm}, % tighter vertical spacing
  boxed title style={
    boxrule=0.4pt,
    colframe=#2!70!black,
    colback=#2!10!white, % lighter title fill
    rounded corners,
    drop shadow=black!15,
    top=0.6mm, bottom=0.6mm, left=1.8mm, right=1.8mm % slimmer padding
  },
  before skip=2mm, after skip=2mm,
  #1
}

\title{iGRPO: Self‑Feedback–Driven LLM Reasoning}

\author{Ali Hatamizadeh\footnote[1]{Project lead. Correspondence to: Ali Hatamizadeh<ahatamizadeh@nvidia.com>.}, Shrimai Prabhumoye, Igor Gitman, Ximing Lu, Seungju Han, Wei Ping, Yejin Choi, Jan Kautz\\
}

\begin{document}

\maketitle

\begin{abstract}
Large Language Models (LLMs) have shown promise in solving complex mathematical problems, yet they still fall short of producing accurate and consistent solutions. Reinforcement Learning (RL) is a framework for aligning these models with task-specific rewards, improving overall quality and reliability. Group Relative Policy Optimization (GRPO) is an efficient, value-function-free alternative to Proximal Policy Optimization (PPO) that leverages group-relative reward normalization. We introduce Iterative Group Relative Policy Optimization (iGRPO), a two-stage extension of GRPO that adds dynamic self-conditioning through model-generated drafts. In Stage 1, iGRPO samples multiple exploratory drafts and selects the highest-reward draft using the same scalar reward signal used for optimization. In Stage 2, it appends this best draft to the original prompt and applies a GRPO-style update on draft-conditioned refinements, training the policy to improve beyond its strongest prior attempt. Under matched rollout budgets, iGRPO consistently outperforms GRPO across base models (e.g., Nemotron-H-8B-Base-8K and DeepSeek-R1 Distilled), validating its effectiveness on diverse reasoning benchmarks. Moreover, applying iGRPO to OpenReasoning-Nemotron-7B trained on AceReason-Math achieves new state-of-the-art results of 85.62\% and 79.64\% on AIME24 and AIME25, respectively. Ablations further show that the refinement wrapper generalizes beyond GRPO variants, benefits from a generative judge, and alters learning dynamics by delaying entropy collapse. These results underscore the potential of iterative, self-feedback-based RL for advancing verifiable mathematical reasoning.
\end{abstract}

\abscontent
% \begin{center} \textbf{Code:} \url{https://github.com/NVlabs/iGRPO} \\ 
% \end{center}

% \begin{figure}[h]
% \centerline{\includegraphics[width=0.665\linewidth]{figures/teaser_igrpo.pdf}}
% \vspace{-.3cm}
% \caption{Comparison of AIME 2024 Pass@1 accuracy across models of varying sizes. 
%     Our model, \textbf{i-Nemotron-7B}, achieves \textbf{85.7\%}, significantly outperforming other 7B--14B counterparts and even surpassing several larger models, 
%     highlighting the efficiency and strong reasoning ability of our approach.}
% \vspace{-.3cm}
% \label{fig:teaser_igrpo}
% \end{figure}

\section{Introduction}
\label{sec:introduction}
\vspace{-.2cm}
Reinforcement Learning (RL) has proven to be successful in improving reasoning capabilities of LLMs by optimizing against task-specific reward signals. Early successes in this direction include RL from Human Feedback (RLHF) for aligning LLMs with human intent, most notably in InstructGPT~\citep{ouyang2022training} and ChatGPT~\citep{achiam2023gpt}, which have demonstrated that incorporating preference-based rewards can dramatically improve both the usability and correctness of model outputs. Recently DeepSeek-R1~\citep{guo2025deepseek} proposed a distinguishing feature which is the so-called zero configuration, wherein the RL process directly enhances the base language model. This breakthrough started several efforts which were targeted at replicating DeepSeek-R1’s methodology or refining its underlying RL mechanisms~\citep{zeng2025simplerl,yu2025dapo,liu2025understanding,cui2025process,hu2025open}.

Yet, in the realm of complex reasoning, RL algorithms typically do not incorporate any form of \emph{feedback or reflection} on the model’s own outputs. Humans, by contrast, rarely solve nontrivial problems in a single pass: they often iterate on initial drafts, identify mistakes, and refine their solutions based on internal feedback~\citep{flower1981cognitive,simon2012architecture,braidotti2019theoretical,flavell1979metacognition,schon2017reflective,polya2014solve}. There is growing evidence that self-feedback mechanisms can bolster multi-step reasoning and the capacity to correct errors~\citep{madaan2023self,shinn2023reflexion}. However, existing RL frameworks do not capitalize on this iterative refinement process, leaving a critical gap between how humans naturally solve problems and how LLMs are typically trained to do so.

In this work, we propose to fill this gap with \emph{Iterative GRPO (iGRPO)} which is a powerful extension of GRPO~\citep{shao2024deepseekmath}. As illustrated in Figure~\ref{fig:pipeline}, Our method operates in two stages. First, we draw multiple candidate completions from the model and compute their relative rewards via the group-based mechanism of GRPO. We then select the highest-scoring draft and this serves as the "first-draft" output of the model. We consider this highest-scoring response as a guide to improve the final output.
Hence, it is provided as a self-feedback to the model.
We feed it back to the model alongside the original prompt. By conditioning on this exemplar, the second stage encourages the model to refine and surpass its own best prior attempt. Notably, this design preserves the efficiency of GRPO while introducing only minimal extra overhead, as iGRPO still relies on the same set of group-based reward signals. In doing so, iGRPO offers a promising avenue for self-guided improvement, enabling LLMs to iteratively improve their reasoning capabilities.

\begin{figure}[t]
  \centering
  \includegraphics[
    width=\linewidth,
    trim=0 0 0 1.06cm,
    clip
  ]{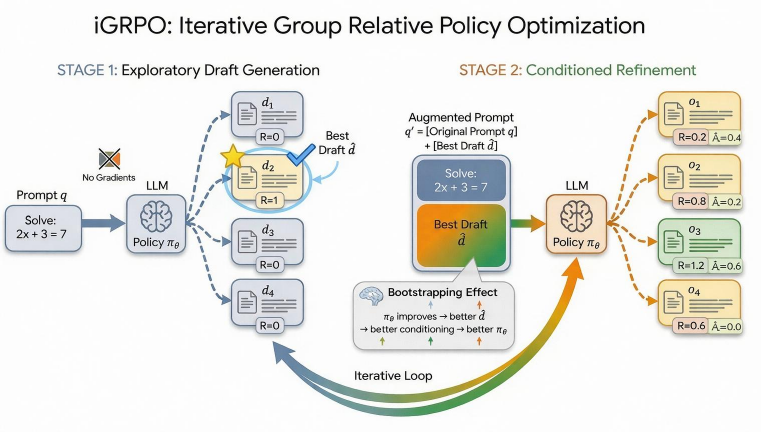}
    \caption{\textbf{Iterative GRPO (iGRPO):} During \textbf{Exploratory Draft Generation}, the model selects a high-scoring ``best draft'' from initial samples and appends it to the prompt for \textbf{Conditioned Refinement}. This augmented context guides the generation of new group-based updates, creating a bootstrapping effect where the policy continuously improves its own conditioning signal to enhance reasoning.}
  \label{fig:pipeline}
\end{figure}

We conduct a series of controlled experiments to compare iGRPO and GRPO under identical 
training conditions, using different base models trained on the Mathematics Aptitude Test of 
Heuristics (MATH)~\citep{hendrycks2021measuring} dataset. Specifically, we evaluate DeepSeek-R1 Distilled~\citep{guo2025deepseek} 
and OpenMath-Nemotron~\citep{moshkov2025aimo} on an extensive array of mathematical reasoning benchmarks, 
including AIME24~\citep{aimo_validation_aime_2024}, AIME25~\citep{opencompass_aime2025}, MATH500~\citep{lightman2023let}, 
AMC23~\citep{aimo_validation_amc_2024}, GSM8K~\citep{gsm8k}, and Minerva Math~\citep{minerva}. 
For models with 7B and 14B parameters, iGRPO consistently outperforms standard GRPO. Moreover, by leveraging iGRPO algorithm with OpenReasoning-Nemotron-7B model~\citep{nvidia_openreasoning_nemotron_7b} on the large-scale \textbf{AceReason‑Math}~\citep{chen2025acereason} dataset~\citep{openr1_math_220k}, we push the state of the art on AIME24 and AIME25 to 85.62\% and 79.64\%, respectively. These findings underscore the effectiveness of incorporating a self-feedback stage into group-based RL optimization, particularly for complex mathematical reasoning tasks.

\section{Related Work}
\label{sec:related_work}

\paragraph{RL for Reasoning.}
Reinforcement learning (RL) has become an important tool for refining large language models on logical and analytical tasks \citep{lambert2024t}. Early self-improvement lines of work, such as STaR-style bootstrapping with verified outcomes and sampling-based selection (as discussed in \citep{lambert2024t}), demonstrate that iteratively leveraging model-generated solutions can improve reasoning behavior. More recent systems scale outcome-driven training substantially \citep{jaech2024openai}, and open-weight efforts such as DeepSeek-R1 report strong reasoning performance under similar training regimes \citep{guo2025deepseek}. Beyond natural language tasks, RL on procedurally generated puzzles \citep{xie2025logic} and settings with limited human demonstrations \citep{wang2025reinforcementlearningreasoninglarge} further highlight the breadth of RL as a mechanism for improving mathematical and logical reasoning.

\paragraph{GRPO and Variants.}
There has also been rapid progress in refining and extending GRPO for large-scale LLM training.
Dr.\ GRPO \citep{liu2025understanding} analyzes sources of bias in GRPO-style token-level objectives and proposes modifications such as removing divisions by sequence length and group-level standard deviation to better match unbiased policy gradients.
DAPO \citep{yu2025dapo} targets long chain-of-thought training through dynamic sampling, decoupled clipping, and reward shaping designed to mitigate instability and reward noise.
GSPO \citep{zheng2025group} instead operates at the sequence level, redefining importance ratios and applying sequence-level clipping to improve stability, especially in Mixture-of-Experts settings.
Whereas these approaches primarily focus on stabilizing or correcting the underlying optimization objective, iGRPO is orthogonal: it introduces a two-stage mechanism that uses externally evaluated best drafts as additional training context, shaping the data distribution seen by the optimizer.

\paragraph{LLM Self-Learning and Self-Improvement.}
LLM self-learning methods aim to improve a model by leveraging feedback signals produced by the model itself or by closely related agents. For example, SPIN and Self-Rewarding Language Models \citep{chen2024self,yuan2401self} use the model (or a closely coupled variant) as an internal evaluator to drive further learning. Other approaches incorporate self-play, verifier-based alignment, or proof-oriented training signals \citep{kirchner2024prover,DBLP:journals/corr/abs-2411-00062}, though unreliable rewards can hinder complex reasoning \citep{lambert2024t}. More specialized self-play extensions, such as SPC \citep{chen2025spcevolvingselfplaycritic} and SPAG \citep{DBLP:conf/nips/ChengHXZDHDL24}, introduce curated tasks or adversarial scenarios to strengthen critique and robustness. Self-Verification \citep{zhang2025incentivizing} unifies problem-solving and generative verification within a single RL framework, enabling inference-time scaling by using the model's own verification scores to reweight or aggregate sampled solutions. Critique-GRPO \citep{zhang2025critique} augments GRPO with natural-language critiques by generating critique-conditioned refinements and optimizing over both initial answers and their refinements in an online RL loop.
While these methods often blur the roles of reward provider and reward recipient within a single agent (or paired agents), \emph{iGRPO} instead uses \emph{externally evaluated} best-prior drafts as an in-context guide for subsequent generations during training. This maintains a clearer separation between the model's generation process and the reward signal, while still leveraging the core self-improvement principle of learning from the model's own outputs.

\section{Methodology}

\subsection{Background: Group Relative Policy Optimization}
\label{sec:background-grpo}

GRPO is a value-function-free variant of proximal policy optimization that leverages \emph{group-based} relative rewards for advantage estimation. Let $\pi_{\theta}$ denote the current policy and $\pi_{\theta_{\text{old}}}$ the old policy. We begin with a pretrained language model $\mathcal{M}$ and a set of training instances $\mathcal{B} = \{(q, a)\}$, where $q$ is a prompt (e.g., a math problem) and $a$ is a reference answer. Given a prompt $q$, GRPO samples a group of $G$ candidate outputs $\{o_1, o_2, \dots, o_G\}$ from $\pi_{\theta_{\text{old}}}$:
\[
    o_i \sim \pi_{\theta_{\text{old}}}(\cdot \mid q) \quad \text{for } i=1,\dots,G.
\]
A reward model then evaluates each sampled output $o_i$, resulting in scores $\{R_1, \dots, R_G\}$. GRPO normalizes these scores within the group to compute the advantage $\hat{A}_{i,t}$ at each token index $t$ of $o_i$:
\[
    \hat{A}_{i,t}
    \;=\;
    \frac{
      R_i \;-\; \mathrm{mean}(\{R_1,\dots,R_G\})
    }{
      \mathrm{std}(\{R_1,\dots,R_G\})
    },
\]
where $\mathrm{mean}(\cdot)$ and $\mathrm{std}(\cdot)$ denote the sample mean and sample standard deviation of the group’s reward scores. If $\mathrm{std}(\{R_1,\dots,R_G\}) = 0$, we set the normalized advantages to $0$ (equivalently, one may add a small constant $\delta$ to the denominator); we apply the same convention in iGRPO. Note that all tokens $t$ in $o_i$ share the same advantage $\hat{A}_{i,t} = \hat{A}_{i}$, reflecting a single scalar reward for each sampled completion.

GRPO then updates the current policy $\pi_{\theta}$ by maximizing a clipped surrogate objective. Let
\[
  r_{i,t}(\theta)
  \;=\;
  \frac{
    \pi_{\theta}(o_{i,t}\mid q,o_{i,<t})
  }{
    \pi_{\theta_{\text{old}}}(o_{i,t}\mid q,o_{i,<t})
  }.
\]

Then the GRPO objective is:
\begin{small}
\begin{align}
\mathcal{J}_{\text{GRPO}}(\theta)
=\;&
\mathbb{E}\Bigl[
  q \sim P(Q),\;
  \{o_i\}_{i=1}^{G} \sim \pi_{\theta_{\text{old}}}
\Bigr]
\nonumber\\
& \times
\frac{1}{G}
\sum_{i=1}^G
\frac{1}{|o_i|}
\sum_{t=1}^{|o_i|}
\Bigl[
  \min\!\Bigl(
    r_{i,t}(\theta)\,\hat{A}_{i,t},\;
    \mathrm{clip}\!\bigl(
      r_{i,t}(\theta),\;1 - \epsilon,\;1 + \epsilon
    \bigr)\,\hat{A}_{i,t}
  \Bigr)
  \;-\;\beta\,\widehat{D}_{\mathrm{KL}}^{(i,t)}
\Bigr],
\label{eq:grpo-objective}
\end{align}
\end{small}
where $\epsilon$ is the PPO clipping parameter, $\beta$ is a regularization coefficient on the KL divergence to a reference policy $\pi_{\text{ref}}$, and $|o_i|$ denotes the token length of completion $o_i$. We use the following non-negative \emph{per-token} estimator $\widehat{D}_{\mathrm{KL}}^{(i,t)}$ as a practical KL penalty \cite{kl_approx}:
\[
\widehat{D}_{\mathrm{KL}}^{(i,t)}
\;=\;
\frac{\pi_{\text{ref}}(o_{i,t}\mid q,o_{i,<t})}
     {\pi_{\theta}(o_{i,t}\mid q,o_{i,<t})}
\;-\;\log\!
\frac{\pi_{\text{ref}}(o_{i,t}\mid q,o_{i,<t})}
     {\pi_{\theta}(o_{i,t}\mid q,o_{i,<t})}
\;-\;1,
\]
which remains guaranteed to be non-negative. This estimator is unbiased for $D_{\mathrm{KL}}(\pi_{\theta}\,\|\,\pi_{\text{ref}})$ when the expectation is taken under samples from $\pi_{\theta}$, and it serves as a convenient sample-based penalty within PPO-style updates. By directly computing group-based advantages instead of estimating value functions, GRPO avoids the overhead of training a separate critic model, making it particularly appealing for large-scale language model fine-tuning.

\subsection{Iterative Group Relative Policy Optimization}
\label{sec:igrpo_method}

We introduce \textbf{Iterative Group Relative Policy Optimization (iGRPO)}, a reinforcement learning algorithm that implements \emph{bootstrapped policy improvement} through dynamic self-conditioning. Unlike standard policy gradient methods that optimize single-shot generations, iGRPO establishes a closed-loop refinement process where the policy learns to systematically improve upon its own best attempts. The key insight is that by coupling exploration (Stage~1) with conditioned exploitation (Stage~2) within each optimization step, the model acquires a generalizable self-improvement capability that compounds in training.

\subsubsection{Motivation: From Static Examples to Dynamic Self-Conditioning}
\label{sec:motivation}

The standard GRPO objective (Eq.~\ref{eq:grpo-objective}) treats each generation as independent, ignoring potentially valuable information from the model's own generation process. In-context learning (ICL) addresses this by conditioning on fixed demonstrations:
\[
\mathcal{J}_{\text{ICL}}(\theta)
=
\mathbb{E}_{q \sim P(Q)}
\left[
  \mathbb{E}_{o \sim \pi_\theta(\cdot|q, e)}
  \left[
    R_\phi(o)
  \right]
\right],
\]
where $e$ is a static example. However, ICL suffers from a fundamental limitation: the conditioning signal $e$ remains fixed throughout training and does not adapt to the evolving policy's capabilities.

iGRPO introduces a different paradigm, \textbf{dynamic self-conditioning}, where the conditioning signal is generated by the policy itself and co-evolves with learning:
\[
\begin{aligned}
\mathcal{J}_{\text{iGRPO}}(\theta)
&=
\mathbb{E}_{q \sim P(Q)}\!\left[
  \mathbb{E}_{o \sim \pi_\theta(\cdot\!\mid\! q'_\theta(q))}\!\left[
    R_\phi(o)
  \right]
\right], \\
q'_\theta(q)
&=
\mathrm{Concat}\!\bigl(q,\hat{d}_\theta(q)\bigr)
\end{aligned}
\]
where $\hat{d}_\theta(q)$ is the best draft generated by the \emph{current} policy (defined formally below). This creates a bootstrapped learning dynamic: as $\pi_\theta$ improves, so does the quality of $\hat{d}_\theta$, which in turn provides increasingly informative conditioning for subsequent generations.

Although the expressions above are written in terms of $\pi_\theta$ for clarity, iGRPO follows the PPO/GRPO convention for stable optimization. At each iteration, we take a snapshot $\pi_{\theta_{\text{old}}}$, sample both Stage~1 drafts and Stage~2 completions from $\pi_{\theta_{\text{old}}}$, and update $\theta$ using importance ratios computed relative to $\pi_{\theta_{\text{old}}}$. Stage~1 is not differentiated through, but the distribution of selected drafts $\hat{d}$ (and thus the self-conditioned prompts $q'$) shifts as $\theta$ evolves across iterations. In addition, dynamic self-conditioning is used only during training. At inference time, we use the trained policy in the standard single-shot manner, generating directly from the original prompt $q$ without any draft generation, conditioning, or specialized selection scheme.

\subsubsection{Algorithmic Framework}
\label{sec:algorithm}

iGRPO operates through two tightly coupled stages within each optimization step. Crucially, only Stage~2 outputs receive gradient updates, while Stage~1 serves as an adaptive exploration mechanism that shapes the optimization landscape.

\paragraph{Stage 1: Exploratory Draft Generation.}
Given a prompt $q$, we sample $N$ candidate drafts from the current policy snapshot:
\[
d_i \sim \pi_{\theta_{\text{old}}}(\cdot \mid q), \quad i = 1, \ldots, N.
\]
Each draft is evaluated using a reward function $R_\phi$, and we identify the highest-scoring draft:
\begin{equation}
\hat{d} = \arg\max_{i \in \{1, \ldots, N\}} R_\phi(d_i).
\label{eq:best_draft}
\end{equation}
This stage performs \emph{implicit curriculum generation}: early in training, $\hat{d}$ may be a weak solution, but as the policy improves, $\hat{d}$ increasingly represents a high-quality attempt that approaches (but does not yet reach) optimal performance.

\paragraph{Stage 2: Conditioned Refinement.}
We form an augmented prompt by appending the best draft immediately after the original prompt:
\begin{equation}
q' = \text{Concat}(q, \hat{d}).
\label{eq:igrpo_prompt_concat}
\end{equation}
We then sample a group of $G$ completions, mirroring the standard GRPO sampling:
\[
o_j \sim \pi_{\theta_{\text{old}}}(\cdot \mid q'), \quad j = 1, \ldots, G.
\]
These completions are scored, and GRPO-style advantage estimation and policy updates are applied exclusively to these Stage~2 outputs. In practice, the concatenation in Eq.~\ref{eq:igrpo_prompt_concat} uses a fixed prompt template, which we provide in the supplementary materials.

\subsubsection{Theoretical Analysis: Bootstrapped Policy Improvement}
\label{sec:theory}

\begin{definition}[Self-Conditioned Prompt Construction]
Given a prompt $q$ and policy $\pi_\theta$, sample $N$ drafts
\[
d_i \sim \pi_\theta(\cdot \mid q), \quad i=1,\ldots,N,
\]
and select the best draft
\[
\hat{d}_\theta(q) = \arg\max_{i \in \{1,\ldots,N\}} R_\phi(d_i).
\]
We then define the self-conditioned prompt by appending the best draft immediately after the original prompt:
\[
q'_\theta(q) = \text{Concat}\bigl(q, \hat{d}_\theta(q)\bigr).
\]
\end{definition}

Unlike static ICL where the conditioning is independent of $\theta$, iGRPO's conditioning is \emph{policy-dependent}: the constructed prompt $q'_\theta(q)$ changes as $\theta$ changes, since $\hat{d}_\theta(q)$ is generated by $\pi_\theta$. This creates a coupled dynamical system.

\begin{proposition}[Progressive Conditioning Quality for Binary Rewards]
Assume the reward is binary, $R_\phi(o) \in \{0,1\}$, and Stage~1 drafts $\{d_i\}_{i=1}^N$ are sampled i.i.d.\ from $\pi_\theta(\cdot|q)$. Let
\[
V_\theta(q) = \mathbb{E}_{o \sim \pi_\theta(\cdot|q)}[R_\phi(o)]
\]
denote the expected reward under policy $\pi_\theta$, which equals the success probability $p_\theta(q) = \Pr[R_\phi(o)=1]$ in the binary case. Then the expected reward of the selected best draft $\hat{d}_\theta(q) = \arg\max_i R_\phi(d_i)$ satisfies
\[
\mathbb{E}\bigl[R_\phi(\hat{d}_\theta(q))\bigr]
=
1 - (1 - V_\theta(q))^N,
\]
which is monotonically increasing in $V_\theta(q)$. Consequently, if optimization increases $V_\theta(q)$, then $\mathbb{E}[R_\phi(\hat{d}_\theta(q))]$ also increases, improving the conditioning quality for subsequent iterations in expectation.
\end{proposition}

\begin{proof}
For binary rewards, $R_\phi(\hat{d}_\theta(q))=1$ if and only if at least one of the $N$ sampled drafts achieves reward $1$. Under i.i.d.\ sampling, this occurs with probability
\[
1 - \Pr\bigl[\forall i,\; R_\phi(d_i)=0\bigr]
=
1 - (1 - V_\theta(q))^N.
\]
Since $R_\phi(\hat{d}_\theta(q)) \in \{0,1\}$, its expectation equals this probability. The function $1-(1-x)^N$ is increasing in $x \in [0,1]$, establishing monotonicity in $V_\theta(q)$.
\end{proof}

This proposition captures the \textbf{bootstrapping effect}: better policies generate better drafts, which provide more informative conditioning, which enables learning better policies. The model does not merely learn to copy the conditioning; it learns a \emph{refinement function} that maps draft attempts to improved solutions.

\subsubsection{Mathematical Formulation}
\label{sec:math_formulation}

Building on the GRPO objective in Eq.~\ref{eq:grpo-objective}, we present the complete iGRPO formulation. Let $\pi_\theta$ denote the policy being optimized, with $\pi_{\theta_{\text{old}}}$ representing the policy snapshot at the start of each iteration for importance sampling.
\paragraph{Stage 1: Draft Selection.}
For each prompt $q$, Stage~1 samples $N$ drafts and selects the best:
\[
\{d_1, \ldots, d_N\} \sim \pi_{\theta_{\text{old}}}(\cdot \mid q), \qquad
\hat{d} = \arg\max_{i \in \{1,\ldots,N\}} R_\phi(d_i).
\]
\paragraph{Stage 2: Conditioned Generation and Advantage Computation.}
We form the augmented prompt by appending the selected draft to $q$ (Eq.~\ref{eq:igrpo_prompt_concat}) and sample $G$ completions:
\[
\{o_1, \ldots, o_G\} \sim \pi_{\theta_{\text{old}}}(\cdot \mid q').
\]
Each completion $o_j$ receives a reward $R_\phi(o_j)$, and advantages are computed via group normalization as in standard GRPO:
\begin{equation}
\hat{A}_j
=
\frac{
  R_\phi(o_j) - \mathrm{mean}(\{R_\phi(o_1), \ldots, R_\phi(o_G)\})
}{
  \mathrm{std}(\{R_\phi(o_1), \ldots, R_\phi(o_G)\})
}.
\label{eq:igrpo_advantage}
\end{equation}
If $\mathrm{std}(\{R_\phi(o_1), \ldots, R_\phi(o_G)\}) = 0$, we set the normalized advantages to $0$ (equivalently, one may add a small constant $\delta$ to the denominator). All tokens $t$ in $o_j$ share the same advantage $\hat{A}_{j,t} = \hat{A}_j$, consistent with the GRPO formulation.

\paragraph{Full Objective.}
The complete iGRPO objective combines both stages:
\begin{align}
\mathcal{J}_{\text{iGRPO}}(\theta)
&=
\mathbb{E}\Bigl[q \sim P(Q)\Bigr]\;
\mathbb{E}\Bigl[
  \underbrace{\{d_i\}_{i=1}^N \sim \pi_{\theta_{\text{old}}}(\cdot \mid q)}_{\text{Stage 1}},\;
  \hat{d} = \arg\max_i R_\phi(d_i),\;
  q' = \mathrm{Concat}(q,\hat{d}),\;
  \underbrace{\{o_j\}_{j=1}^G \sim \pi_{\theta_{\text{old}}}(\cdot \mid q')}_{\text{Stage 2}}
\Bigr]
\nonumber\\
&\quad \times
\frac{1}{G}
\sum_{j=1}^{G}
\frac{1}{|o_j|}
\sum_{t=1}^{|o_j|}
\Bigl[
  \min\!\Bigl(
    r_{j,t}(\theta)\,\hat{A}_{j},\;
    \mathrm{clip}\!\bigl(
      r_{j,t}(\theta),\,1-\epsilon,\,1+\epsilon
    \bigr)\,\hat{A}_{j}
  \Bigr)
  \;-\;\beta\,\widehat{D}_{\mathrm{KL}}^{(j,t)}
\Bigr],
\label{eq:igrpo_objective}
\end{align}
where the importance sampling ratio is computed with respect to the augmented prompt:
\[
r_{j,t}(\theta)
=
\frac{\pi_{\theta}(o_{j,t} \mid q', o_{j,<t})}
     {\pi_{\theta_{\text{old}}}(o_{j,t} \mid q', o_{j,<t})},
\]
and we use the same non-negative per-token KL penalty as in Section~\ref{sec:background-grpo}, computed under the augmented prompt:
\[
\widehat{D}_{\mathrm{KL}}^{(j,t)}
=
\frac{\pi_{\text{ref}}(o_{j,t} \mid q', o_{j,<t})}
     {\pi_{\theta}(o_{j,t} \mid q', o_{j,<t})}
-
\log
\frac{\pi_{\text{ref}}(o_{j,t} \mid q', o_{j,<t})}
     {\pi_{\theta}(o_{j,t} \mid q', o_{j,<t})}
-1.
\]
The key structural difference from vanilla GRPO (Eq.~\ref{eq:grpo-objective}) is that all Stage~2 completions are conditioned on the augmented prompt $q'$, formed by appending the best Stage~1 draft immediately after the original prompt $q$. This provides a self-feedback signal that encourages the policy to refine its solutions beyond its strongest initial attempt.

\paragraph{Reward Function.}
Following standard practice for verifiable reasoning tasks, we employ a rule-based reward function:
\begin{equation}
R_\phi(o) = \mathbf{1}\left[\texttt{extract}(o) = a\right],
\label{eq:reward}
\end{equation}
where $\texttt{extract}(\cdot)$ parses the final answer from the completion and $a$ is the ground-truth reference answer from the training instance $(q, a) \in \mathcal{B}$. This binary reward, combined with group normalization (Eq.~\ref{eq:igrpo_advantage}), provides sufficient signal for distinguishing solution quality within each prompt group.
The complete iGRPO procedure is presented in Algorithm~\ref{alg:igrpo_full}.

\begin{algorithm}[t]
\caption{Iterative Group Relative Policy Optimization (iGRPO)}
\label{alg:igrpo_full}
\begin{algorithmic}[1]
\Require Pretrained model $\mathcal{M}$, training set $\mathcal{B} = \{(q, a)\}$, reward function $R_\phi$, draft count $N$, group size $G$, clipping parameter $\epsilon$, KL coefficient $\beta$, iterations $I$, batch size $S$
\State Initialize $\pi_\theta \gets \mathcal{M}$, $\pi_{\text{ref}} \gets \mathcal{M}$
\For{iteration $= 1, \ldots, I$}
    \State $\pi_{\theta_{\text{old}}} \gets \pi_\theta$ \Comment{Snapshot policy for sampling}
    \State Sample batch $\{(q^{(k)}, a^{(k)})\}_{k=1}^S \sim \mathcal{B}$
    \For{$k = 1, \ldots, S$}
        \State \textcolor{blue}{\texttt{// Stage 1: Draft Generation}}
        \State Sample drafts $\{d_i^{(k)}\}_{i=1}^N \sim \pi_{\theta_{\text{old}}}(\cdot \mid q^{(k)})$
        \State Compute rewards $\{R_\phi(d_i^{(k)})\}_{i=1}^N$ using reference answer $a^{(k)}$
        \State Select best draft: $\hat{d}^{(k)} \gets \arg\max_i R_\phi(d_i^{(k)})$
        \State \textcolor{blue}{\texttt{// Stage 2: Conditioned Refinement}}
        \State Construct augmented prompt: $q'^{(k)} \gets \text{Concat}(q^{(k)}, \hat{d}^{(k)})$
        \State Sample refinements $\{o_j^{(k)}\}_{j=1}^G \sim \pi_{\theta_{\text{old}}}(\cdot \mid q'^{(k)})$
        \State Compute rewards $\{R_\phi(o_j^{(k)})\}_{j=1}^G$
        \State Compute advantages $\{\hat{A}_j^{(k)}\}$ via Eq.~\ref{eq:igrpo_advantage}
    \EndFor
    \State \textcolor{blue}{\texttt{// Policy Update}}
    \State $\theta \gets \theta + \eta \nabla_\theta \mathcal{J}_{\text{iGRPO}}(\theta)$ \Comment{See Eq.~\ref{eq:igrpo_objective}}
\EndFor
\State \Return $\pi_\theta$
\end{algorithmic}
\end{algorithm}

\subsubsection{Computational Analysis}
\label{sec:complexity}

Although iGRPO uses two stages, its compute is primarily controlled by the \emph{total} number of sampled completions per prompt. With a fixed sampling budget, iGRPO can be run at essentially the same dominant generation cost as GRPO.

Let $C_{\text{gen}}$ denote the cost of producing one sampled completion (including prompt encoding, autoregressive decoding, and reward evaluation). Let $G_{\text{GRPO}}$ denote the number of completions sampled per prompt in standard GRPO.

\paragraph{Baseline GRPO Cost.}
Standard GRPO samples $G_{\text{GRPO}}$ completions per prompt, giving per-prompt rollout cost
\[
C_{\text{GRPO}} \;\approx\; G_{\text{GRPO}} \, C_{\text{gen}}.
\]

\paragraph{iGRPO Rollout Cost.}
iGRPO samples $N$ Stage~1 drafts and $G$ Stage~2 refinements per prompt (Section~\ref{sec:igrpo_method}), so
\[
C_{\text{iGRPO}} \;\approx\; (N + G)\, C_{\text{gen}},
\qquad
\frac{C_{\text{iGRPO}}}{C_{\text{GRPO}}}
\;=\;
\frac{N+G}{G_{\text{GRPO}}}.
\]
In our experiments, we keep the sampling budget fixed by setting
\[
N + G \;=\; G_{\text{GRPO}},
\]
so iGRPO redistributes the same number of rollouts across Stage~1 and Stage~2 rather than increasing them. For example, with $G_{\text{GRPO}}=16$ in GRPO, we use $N=8$ and $G=8$ in iGRPO, yielding
\[
C_{\text{iGRPO}}
\;\approx\;
(N+G)\,C_{\text{gen}}
\;=\;
G_{\text{GRPO}}\,C_{\text{gen}}
\;\approx\;
C_{\text{GRPO}},
\]
i.e., the dominant generation cost is comparable to GRPO. 

\begin{table*}[t]
\centering
\setlength{\tabcolsep}{6pt}
\newcommand{\best}[1]{\textbf{#1}}
\newcommand{\second}[1]{\underline{#1}}

\caption{Performance comparison of 7B, 8B, and 14B models across multiple mathematical reasoning benchmarks.
Bold indicates the best and underlined the second best per column within each parameter bucket (7B, 8B, or 14B).
Our method \emph{iGRPO} rows are lightly shaded.}
\label{tab:exps_main}

\resizebox{\textwidth}{!}{%
\begin{tabular}{lrrrrrrr}
\toprule
\rowcolor{gray!15}
\textbf{Model} & \textbf{AIME25} & \textbf{AIME24} & \textbf{MATH500} & \textbf{AMC} & \textbf{GSM8K} & \textbf{Minerva} & \textbf{Avg} \\
\midrule
Nemotron-H-8B-Base-8K                   & 6.20 & 8.65 & 61.23 & 43.21 & 41.02 & 17.60 & 29.65 \\
Nemotron-H-8B-Base-8K + GRPO            & 7.78 & 9.01 & 73.13 & 45.10 & 81.93 & 29.56 & 41.08 \\
Nemotron-H-8B-Base-8K + Self-Verification~\citep{zhang2025incentivizing}
                                        & \second{8.50} & \second{9.25} & 75.60 & 46.50 & 86.20 & 31.10 & 42.86 \\
Nemotron-H-8B-Base-8K + Critique-GRPO~\citep{zhang2025critique}
                                        & 8.42 & 9.15 & \second{76.05} & \second{46.80} & \second{88.40} & \second{31.50} & \second{43.39} \\
\rowcolor{gray!10}
\textbf{Nemotron-H-8B-Base-8K + iGRPO}  & \best{9.17} & \best{9.56} & \best{78.80} & \best{48.75} & \best{91.26} & \best{32.72} & \best{45.04} \\

\midrule
DeepSeek-R1-Distill-Qwen-7B~\citep{guo2025deepseek}  & 38.60 & 54.40 & 92.80 & 90.00 & 92.00 & 39.10 & 61.93 \\
DeepSeek-R1-Distill-Qwen-7B + GRPO                   & 38.90 & 55.00 & 93.25 & 90.00 & \second{92.12} & \second{40.44} & 68.29 \\
DeepSeek-R1-Distill-Qwen-7B + Self-Verification~\citep{zhang2025incentivizing}
                                                     & 39.45 & \second{55.80} & \second{93.50} & 92.50 & 92.20 & 41.00 & 69.08 \\
DeepSeek-R1-Distill-Qwen-7B + Critique-GRPO~\citep{zhang2025critique}
                                                     & \second{39.60} & 55.65 & 93.45 & \second{92.80} & \second{92.25} & \second{41.10} & \second{69.14} \\
\rowcolor{gray!10}
\textbf{DeepSeek-R1-Distill-Qwen-7B + iGRPO (ours)}  & \best{40.16} & \best{56.30} & \best{93.80} & \best{95.00} & \best{92.42} & \best{41.54} & \best{69.87} \\
\midrule
OpenMath-Nemotron-7B~\citep{moshkov2025aimo}         & 61.18 & 73.28 & 95.55 & \second{95.00} & 90.52 & 33.46 & 74.83 \\
OpenMath-Nemotron-7B + GRPO                          & \second{61.32} & \second{73.62} & \second{95.60} & \second{95.00} & \second{90.60} & \second{34.00} & \second{75.02} \\
\rowcolor{gray!10}
\textbf{OpenMath-Nemotron-7B + iGRPO (ours)}         & \best{62.45} & \best{74.79} & \best{96.00} & \best{97.50} & \best{90.75} & \best{34.94} & \best{76.07} \\

\midrule
DeepSeek-R1-Distill-Qwen-14B~\citep{guo2025deepseek}           & 42.10 & 58.93 & 93.10 & 90.00 & 93.10 & 45.59 & 70.47 \\
DeepSeek-R1-Distill-Qwen-14B + GRPO                            & \second{43.70} & \second{60.26} & \second{93.10} & \second{91.20} & \second{93.50} & \second{46.00} & \second{71.29} \\
\rowcolor{gray!10}
\textbf{DeepSeek-R1-Distill-Qwen-14B + iGRPO (ours)}           & \best{45.52} & \best{64.06} & \best{94.00} & \best{93.45} & \best{94.00} & \best{47.06} & \best{73.02} \\
\midrule
OpenMath-Nemotron-14B~\citep{moshkov2025aimo}                  & 61.18 & 73.28 & 95.55 & \second{95.00} & 94.01 & 33.46 & 75.41 \\
OpenMath-Nemotron-14B + GRPO                                   & \second{64.53} & \second{74.79} & \second{96.00} & \second{95.00} & \second{94.40} & \second{35.70} & \second{76.73} \\
\rowcolor{gray!10}
\textbf{OpenMath-Nemotron-14B + iGRPO (ours)}                  & \best{65.57} & \best{76.72} & \best{96.70} & \best{97.50} & \best{94.69} & \best{36.76} & \best{78.00} \\

\bottomrule
\end{tabular}%
}
\end{table*}
\section{Experiments}
\label{sec:experiments}
\subsection{Setup}
Our training data includes two datasets: MATH~\citep{hendrycks2021measuring} 
(7,500 step-by-step competition problems) and AceReason-Math~\citep{chen2025acereason} (9,400 problems).
All models are trained for one epoch, with a KL divergence loss coefficient of 0 and no entropy regularization. 
We use a learning rate of $1\times10^{-6}$ with a cosine schedule, generating 
eight completions per prompt (halved in each stage for iGRPO). The maximum prompt length is 1,024 tokens across all datasets. 
For all experiments, the completion length is capped at 4,096 tokens and the batch size at 1024. 
We evaluate model performance on well-known mathematical benchmarks, including AIME24/AIME25, MATH500, AMC23, GSM8K, and Minerva Math. 
For all benchmarks we report Pass@1 accuracy. However, for AIME24/AIME25 the reported value is averaged over 64 runs to ensure robustness. 
For other benchmarks, an average of 8 runs are reported. For evaluations, we use NeMo-Skills framework\footnote{\url{https://github.com/NVIDIA/NeMo-Skills}} 
with decoding parameters such as a temperature of 0.6, top-p of 0.95, and generation length of 65,000. Additional training and evaluation details are provided in the supplementary materials.

\subsection{Results}
\label{sec:controlled}

Our goal in this section is to isolate the practical benefit of adding a self-feedback refinement stage to GRPO. Concretely, we ask whether the two-stage training signal in iGRPO translates into better \emph{verifiable} mathematical reasoning, and whether those gains persist across model families and scales.

\paragraph{Controlled study with matched sampling budget.}
Table~\ref{tab:exps_main} compares \emph{iGRPO} against vanilla GRPO and two recent self-improvement baselines, Self-Verification~\citep{zhang2025incentivizing} and Critique-GRPO~\citep{zhang2025critique}, across model families and parameter scales (7B, 8B, 14B).
All methods share the same training protocol and the same rule-based reward (Eq.~\ref{eq:reward}).
To ensure a fair compute comparison, we keep the total sampling budget fixed at eight completions per prompt for every method; iGRPO redistributes this budget across stages by using Stage~1 drafts to select self-feedback and Stage~2 completions to perform the GRPO-style update.
As a result, differences in Table~\ref{tab:exps_main} reflect the effect of \emph{conditioning on the best draft} rather than the effect of more sampling.

\paragraph{Generalist 8B model: the largest gains from self-feedback.}
For \textbf{Nemotron-H-8B-Base-8K}, vanilla GRPO raises the macro-average from $29.65\%$ to $41.08\%$.
Self-Verification and Critique-GRPO further improve the average to $42.86\%$ and $43.39\%$.
iGRPO performs best at $45.04\%$, which is \textbf{+3.96} points over GRPO and \textbf{+1.65} points over the strongest self-improvement baseline.
The improvements are most visible on the benchmarks that most strongly penalize near-miss reasoning, including AIME25 ($9.17\%$) and Minerva ($32.72\%$), and iGRPO also reaches the highest GSM8K accuracy ($91.26\%$).
This is a setting where avoiding additional self-judgment tasks matters: iGRPO does not require the model to generate critiques or verification rationales; it provides the best draft as a direct, high-signal scaffold and trains the policy to refine beyond it.

\paragraph{Stronger 7B distilled reasoner: consistent gains concentrated on multi-step tasks.}
For \textbf{DeepSeek-R1-Distill-Qwen-7B}, the base model is already strong ($61.93\%$ average), and GRPO improves it to $68.29\%$.
Self-Verification and Critique-GRPO reach $69.08\%$ and $69.14\%$.
iGRPO remains best overall at $69.87\%$.
The gains concentrate on multi-step benchmarks where a mostly-correct attempt can fail due to a small late error (e.g., AIME24 at $56.30\%$ and AMC at $95.00\%$).
This pattern fits the iGRPO mechanism: Stage~1 increases the chance that a strong reasoning trajectory appears in the context, and Stage~2 learns to reliably ``finish the job'' under the same verifiable reward.

\paragraph{Math-specialized 7B model: improvements persist when the base is already strong.}
For \textbf{OpenMath-Nemotron-7B}, the base model starts at $74.83\%$ average.
GRPO yields a small change to $75.02\%$, suggesting limited headroom from standard one-shot RL updates in this regime.
iGRPO increases performance to $76.07\%$, with the most notable gains on the harder benchmarks, including AIME24 (from $73.28\%$ to $74.79\%$) and AMC (from $95.00\%$ to $97.50\%$).
When drafts are already high-quality, the value of iGRPO is less about discovering a viable solution mode and more about systematically reinforcing the most reliable completion patterns.

\paragraph{Scaling to 14B parameters: benefits persist on complex reasoning.}
At the 14B scale, iGRPO continues to improve on GRPO across both model families.
For \textbf{DeepSeek-R1-Distill-Qwen-14B}, the macro-average increases from $71.29\%$ (GRPO) to $73.02\%$ (iGRPO), with a large gain on AIME24 from $60.26\%$ to $64.06\%$.
For \textbf{OpenMath-Nemotron-14B}, the macro-average improves from $76.73\%$ (GRPO) to $78.00\%$ (iGRPO), including gains on AIME25 from $64.53\%$ to $65.57\%$ and on AIME24 from $74.79\%$ to $76.72\%$.
While margins naturally shrink as models become stronger, the improvement pattern remains stable: iGRPO helps larger models correct residual errors that survive strong pretraining and standard RL fine-tuning, especially on long-horizon competition benchmarks.

\paragraph{Competitiveness against critique-style objectives.}
Self-Verification and Critique-GRPO are strong baselines, but they ask the model to allocate capacity to additional behaviors (verifying, critiquing, or producing auxiliary text) that are only indirectly optimized by the outcome reward.
iGRPO keeps the training loop tightly aligned with the verifiable objective while still capturing the benefit of iteration: Stage~1 supplies a high-quality draft as an explicit conditioning signal, and Stage~2 directly optimizes a refinement policy that can surpass that draft.
Empirically, this design yields consistent improvements across model families and scales, with the strongest gains on benchmarks that are sensitive to small long-horizon reasoning mistakes.

\subsection{Generalization to a Stronger Base and Harder Dataset}
\label{sec:openreasoning}

All prior controlled comparisons were trained on the 7{,}500-problem MATH set.
We\begin{wrapfigure}{r}{0.52\columnwidth}
  \centering
  \includegraphics[width=0.50\columnwidth]{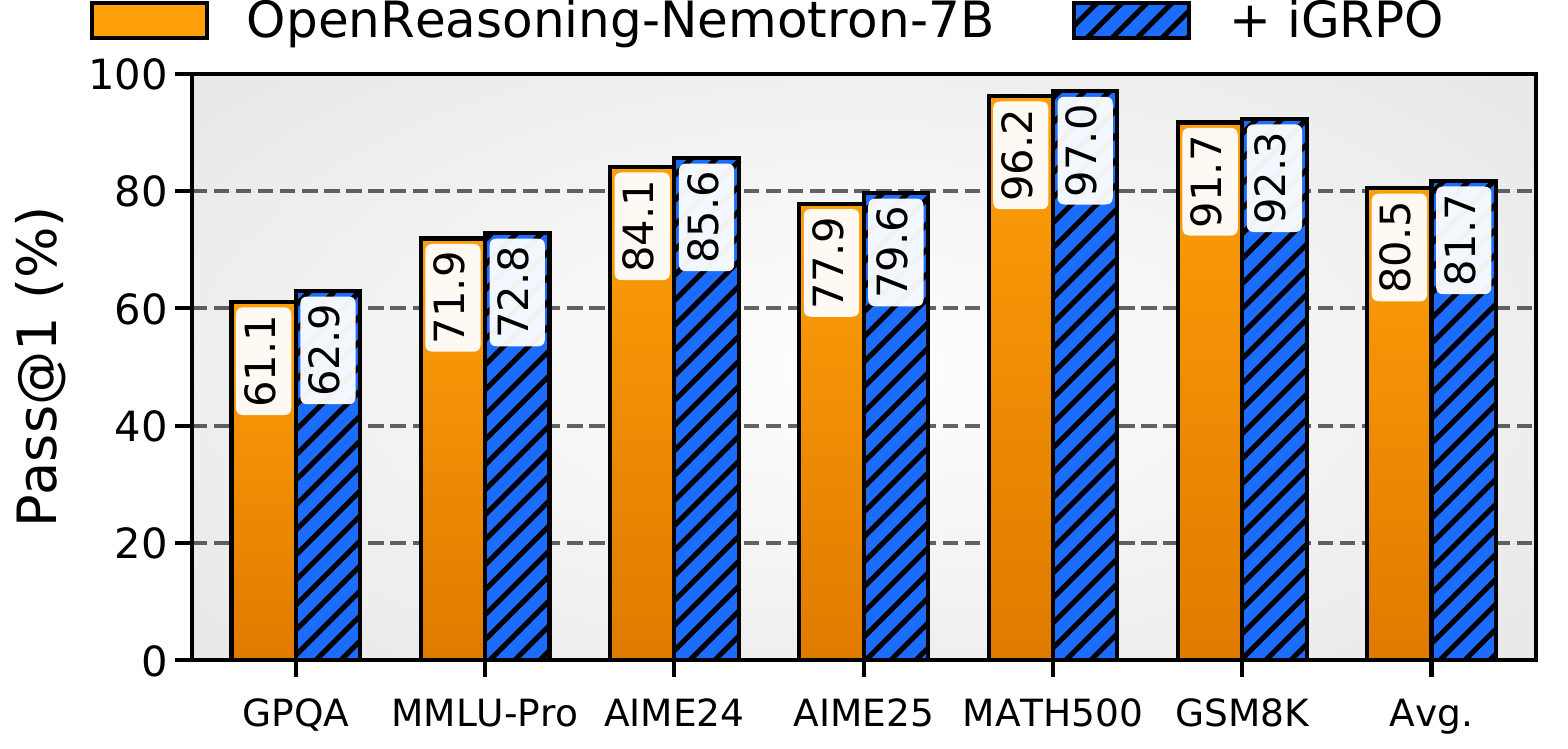}
  \caption{Pass@1 results for OpenReasoning-Nemotron-7B with and without iGRPO. Improvements appear not only on math but also on general reasoning tasks such as MMLU-Pro and GPQA.}
  \label{fig:openreasoning_bars}
\end{wrapfigure}
 next evaluate iGRPO in a harder regime by both strengthening the initialization and shifting to a more challenging training distribution. We start from \textbf{OpenReasoning-Nemotron-7B} and train on \textbf{AceReason-Math}~\citep{chen2025acereason}, keeping the iGRPO configuration identical to Section~\ref{sec:experiments}; Figure~\ref{fig:openreasoning_bars} summarizes the results. iGRPO improves performance across all benchmarks, with the strongest gains on AIME24/25 (\textbf{+1.52}/\textbf{+1.78}) and clear transfer beyond math to GPQA (\textbf{+1.84}) and MMLU-Pro (\textbf{+0.91}), yielding a \textbf{+1.23} overall average improvement.
This indicates iGRPO remains effective with a stronger base model and harder training data, and that iterative self-feedback improves broadly useful refinement behaviors rather than only math-specific patterns.

\section{Ablation}

\paragraph{Beyond GRPO: self-feedback as a reusable refinement wrapper.}
\begin{wraptable}{r}{0.45\textwidth}
\vspace{-1em}
\centering
\footnotesize
\setlength{\tabcolsep}{5pt}
\renewcommand{\arraystretch}{0.8}
\begin{tabular}{lccc}
\toprule
\textbf{Method} & \textbf{Base Avg} & \textbf{+ iGRPO Avg} & \textbf{$\Delta$} \\
\midrule
DAPO & 69.74 & \textbf{70.93} & $\boldsymbol{+1.19}$ \\
GSPO & 69.20 & \textbf{70.31} & $\boldsymbol{+1.11}$ \\
\bottomrule
\end{tabular}
\caption{\footnotesize
\textbf{Beyond GRPO:} adding the same self-feedback refinement layer to other methods consistently improves the average Pass@1.}
\label{tab:beyond_grpo}
\vspace{-1em}
\end{wraptable}
Our two-stage procedure is not tied to the GRPO objective itself: it can be layered on top of other
group-based PPO variants by using Stage~1 to select a high-reward draft and Stage~2 to perform the
base optimizer’s update on self-conditioned prompts.
Table~\ref{tab:beyond_grpo} applies this wrapper to DAPO and GSPO under matched rollout budgets
(same total sampled completions per prompt, same reward, and identical training and evaluation settings).
In both cases, self-feedback yields a consistent \textbf{+1.1 to +1.2} macro-average improvement, indicating that
the gains primarily stem from the refinement interface rather than GRPO-specific details.

\paragraph{Generative judge study.}
\vspace{0.5em}
\begin{wraptable}{r}{0.48\textwidth}
\vspace{-1em}
\centering
\footnotesize
\setlength{\tabcolsep}{3pt}
\renewcommand{\arraystretch}{0.7}
\begin{tabular}{lccc}
\toprule
\textbf{Benchmark} &
\textbf{Rule-based} &
\textbf{GPT-5 Judge} &
$\boldsymbol{\Delta}$ \\
\midrule
AIME25   & 40.16 & \textbf{41.12} & $\boldsymbol{+0.96}$ \\
AIME24   & 56.30 & \textbf{57.45} & $\boldsymbol{+1.15}$ \\
MATH500  & 93.80 & \textbf{94.20} & $\boldsymbol{+0.40}$ \\
AMC      & 95.00 & \textbf{96.25} & $\boldsymbol{+1.25}$ \\
GSM8K    & 92.42 & \textbf{92.95} & $\boldsymbol{+0.53}$ \\
Minerva  & 41.54 & \textbf{42.88} & $\boldsymbol{+1.34}$ \\
\midrule
Average  & 69.87 & \textbf{70.81} & $\boldsymbol{+0.94}$ \\
\bottomrule
\end{tabular}
\caption{\footnotesize
Effect of replacing the rule-based reward with a GPT-5 judge in iGRPO. Base model is DeepSeek R1 Distill Qwen 7B.}
\label{tab:igrpo_generative_judge}
\vspace{-1em}
\end{wraptable}

Because iGRPO only needs a scalar reward to (i) rank Stage~1 drafts and (ii) compute Stage~2 group-normalized advantages, we can swap the binary outcome checker (Eq.~\ref{eq:reward}) for a generative judge. On DeepSeek-R1-Distill-Qwen-7B trained on MATH, GPT-5 scoring each solution in $[0,1]$ improves Pass@1 on all six benchmarks and increases the average from 69.87 to 70.81 (\textbf{+0.94}; Table~\ref{tab:igrpo_generative_judge}). The largest gains on AIME24/25 and Minerva are consistent with partial credit for near-miss traces, which lets them survive Stage~1 selection and be refined into correct answers in Stage~2.

\paragraph{Entropy analysis.}
\begin{wrapfigure}{r}{0.50\columnwidth}
\vspace{-1em}
\centering
\begin{tikzpicture}
\begin{axis}[
    width=0.48\columnwidth,
    height=4.2cm,
    xlabel={Training steps (\%)},
    ylabel={Entropy (nats)},
    xmin=0, xmax=100,
    ymin=0.35, ymax=2.50,
    ytick distance=0.5,
    xtick distance=20,
    grid=major,
    grid style={dashed, gray!25},
    tick label style={font=\footnotesize},
    label style={font=\footnotesize},
    legend style={
        font=\footnotesize,
        fill=white,
        fill opacity=0.85,
        draw=gray!40,
        rounded corners=1pt,
        at={(0.98,0.98)},
        anchor=north east,
        cells={anchor=west},
        inner xsep=3pt,
        inner ysep=2pt,
    },
    clip=false,
]
% iGRPO
\addplot[color=myteal, line width=1.2pt, smooth, tension=0.7]
coordinates {(0,2.45) (15,0.8) (30,0.48) (60,0.46) (100,0.44)};
\addlegendentry{iGRPO}

% GRPO
\addplot[color=myorange, line width=1.2pt, dashed, smooth, tension=0.7]
coordinates {(0,2.45) (10,0.6) (30,0.42) (60,0.42) (100,0.42)};
\addlegendentry{GRPO}
\end{axis}
\end{tikzpicture}
\caption{\footnotesize
\textbf{Entropy dynamics.} Per-token policy entropy during training. iGRPO maintains higher mid-training entropy than GRPO, indicating sustained exploration before convergence.}
\label{fig:entropy_plot}
\vspace{-1em}
\end{wrapfigure}
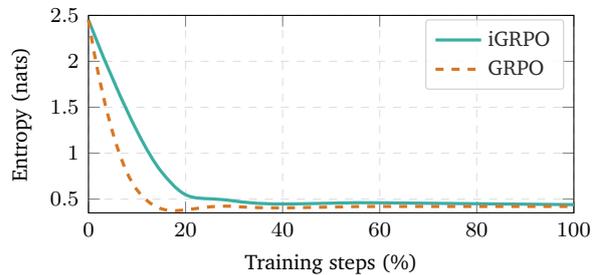

To probe how self-feedback alters learning dynamics, we track the per-token Shannon entropy of the policy during RL. For a decoding step $t$ with context $h_t$ and vocabulary $\mathcal{V}$, we measure entropy in nats as
\begin{equation}
\mathcal{H}\bigl(\pi_\theta(\cdot \mid h_t)\bigr)
= - \sum_{w \in \mathcal{V}} \pi_\theta(w \mid h_t)\,\ln \pi_\theta(w \mid h_t),
\end{equation}
and in practice compute it from the log-softmax of the logits (base $e$), then average over all valid completion tokens in the batch to obtain a single scalar per training step. On DeepSeek-R1-Distill-Qwen-7B trained on MATH, both methods start at $2.45$ nats, but GRPO collapses rapidly ($0.60$ at $10\%$, $0.42$ by $30\%$) and then stays flat through the end. In contrast, iGRPO decays more gradually ($0.80$ at $15\%$, $0.48$ at $30\%$) and remains slightly higher through mid-training ($0.46$ at $60\%$) before converging near GRPO ($0.44$ vs.\ $0.42$ at $100\%$). This suggests iGRPO delays premature mode collapse: conditioning on the best draft encourages refinement around a strong scaffold while preserving alternative continuations long enough to recover from near-miss reasoning traces. Since final entropies are close, the gains are better explained by sustained mid-training exploration rather than higher randomness at convergence.

\section{Conclusion}
\label{sec:conclusion}

We introduced \textbf{Iterative Group Relative Policy Optimization (iGRPO)}, a simple and effective extension of GRPO that injects an explicit \emph{self-feedback} signal into outcome-driven RL for reasoning. iGRPO replaces single-shot optimization with a two-stage loop: Stage~1 samples multiple drafts and selects the highest-reward completion as feedback; Stage~2 conditions on this best draft and applies a standard GRPO-style update on refinements. This yields \emph{dynamic self-conditioning}: the training context automatically improves as the policy improves, creating a bootstrapped refinement behavior that is absent from conventional group-based objectives. We also provided a clean theoretical characterization of the bootstrapping effect under binary rewards, showing that the expected quality of the selected draft increases monotonically with the policy's success probability.

Empirically, iGRPO consistently improves verifiable math reasoning across model families and scales under matched rollout budgets. Across 7B, 8B, and 14B backbones, iGRPO outperforms vanilla GRPO as well as strong self-improvement baselines that rely on critique or verification behaviors (Table~\ref{tab:exps_main}). In a harder generalization setting, training OpenReasoning-Nemotron-7B on AceReason-Math with iGRPO yields new best results on AIME24/AIME25 (85.62\%/79.64\%) and transfers gains beyond math to GPQA and MMLU-Pro (Figure~\ref{fig:openreasoning_bars}). Our ablations further suggest that the benefit primarily comes from the refinement interface itself: the same two-stage wrapper improves other group-based PPO variants (Table~\ref{tab:beyond_grpo}), remains compatible with richer scalar rewards such as a generative judge (Table~\ref{tab:igrpo_generative_judge}), and measurably alters learning dynamics by delaying premature entropy collapse (Figure~\ref{fig:entropy_plot}).

\bibliography{iclr2026_conference}
\bibliographystyle{iclr2026_conference}

\newpage
\section*{Appendix}
\appendix
\renewcommand{\thesection}{\Alph{section}}
\renewcommand\thefigure{S.\arabic{figure}}
\setcounter{figure}{0}
\renewcommand\thetable{S.\arabic{table}}
\setcounter{table}{0}

\section{Policy Gradient Derivation for iGRPO}
\label{sec:appendix_igrpo_policy_gradient}

In this sectin, we derive the iGRPO policy gradient update consistent with the methodology in
Section~\ref{sec:igrpo_method} and the full objective in Eq.~\eqref{eq:igrpo_objective}.
Throughout, $\pi_\theta$ is the trainable policy, $\pi_{\theta_{\mathrm{old}}}$ is the frozen
snapshot used for sampling (PPO/GRPO convention), $\pi_{\mathrm{ref}}$ is the reference policy,
and $R_\phi(\cdot)$ is a scalar (verifier) reward. Importantly, iGRPO uses two stages per
optimization step, but gradients are applied only to Stage~2 tokens.

\subsection{Two-stage sampling and the induced self-conditioned prompt distribution}

For each prompt $q \sim P(Q)$, iGRPO constructs an augmented prompt by first sampling
Stage~1 drafts and selecting the best:
\begin{align}
d_i &\sim \pi_{\theta_{\mathrm{old}}}(\cdot \mid q), \quad i=1,\ldots,N, \\
\hat d &= \arg\max_{i \in \{1,\ldots,N\}} R_\phi(d_i), \\
q' &= \mathrm{Concat}(q,\hat d).
\end{align}
This defines an \emph{implicit} sampling distribution over augmented prompts $q'$ induced by
$\pi_{\theta_{\mathrm{old}}}$ and the $\arg\max$ selection. Within a single PPO-style update,
$q'$ is treated as part of the sampled context (no differentiation through Stage~1 and no
differentiation through the $\arg\max$). Across iterations, the distribution over $q'$ shifts
because $\pi_{\theta_{\mathrm{old}}}$ changes.

Given the augmented prompt $q'$, Stage~2 samples a group of $G$ completions:
\begin{equation}
o_j \sim \pi_{\theta_{\mathrm{old}}}(\cdot \mid q'), \quad j=1,\ldots,G.
\end{equation}
Let each completion be a token sequence $o_j = (o_{j,1},\ldots,o_{j,|o_j|})$ with factorization
\begin{equation}
\pi_\theta(o_j \mid q') \;=\; \prod_{t=1}^{|o_j|} \pi_\theta(o_{j,t} \mid q', o_{j,<t}).
\label{eq:token_factorization}
\end{equation}

\subsection{From the self-conditioned expected reward to a REINFORCE-style gradient}

Fix an augmented prompt $q'$. Consider the (conceptual) self-conditioned objective
\begin{equation}
\mathcal{J}( \theta \mid q')
\;=\;
\mathbb{E}_{o \sim \pi_\theta(\cdot \mid q')}\bigl[ R_\phi(o) \bigr].
\label{eq:conceptual_objective}
\end{equation}
By the score-function (REINFORCE) identity,
\begin{align}
\nabla_\theta \mathcal{J}(\theta \mid q')
&=
\mathbb{E}_{o \sim \pi_\theta(\cdot \mid q')}
\Bigl[
R_\phi(o)\,\nabla_\theta \log \pi_\theta(o \mid q')
\Bigr] \\
&=
\mathbb{E}_{o \sim \pi_\theta(\cdot \mid q')}
\Biggl[
R_\phi(o)\,
\sum_{t=1}^{|o|}
\nabla_\theta \log \pi_\theta(o_t \mid q', o_{<t})
\Biggr],
\label{eq:reinforce_expand}
\end{align}
where the second line uses Eq.~\eqref{eq:token_factorization}.
Introducing any baseline $b(q')$ that does not depend on the sampled tokens preserves
unbiasedness and yields an advantage-weighted gradient:
\begin{equation}
\nabla_\theta \mathcal{J}(\theta \mid q')
=
\mathbb{E}_{o \sim \pi_\theta(\cdot \mid q')}
\Biggl[
\bigl(R_\phi(o)-b(q')\bigr)
\sum_{t=1}^{|o|}
\nabla_\theta \log \pi_\theta(o_t \mid q', o_{<t})
\Biggr].
\label{eq:reinforce_baseline}
\end{equation}

\subsection{Group-relative advantage used in iGRPO}

In iGRPO, the baseline is estimated per prompt by sampling a \emph{group} of $G$ completions
$\{o_j\}_{j=1}^G$ and normalizing rewards within the group. Define
\begin{equation}
R_j \;=\; R_\phi(o_j), \qquad
\overline{R} \;=\; \mathrm{mean}(\{R_1,\ldots,R_G\}), \qquad
s_R \;=\; \mathrm{std}(\{R_1,\ldots,R_G\}),
\end{equation}
and the iGRPO advantage
\begin{equation}
\hat A_j
\;=\;
\frac{R_j - \overline{R}}{s_R},
\qquad
\text{with the convention } \hat A_j = 0 \text{ if } s_R = 0.
\label{eq:appendix_advantage}
\end{equation}
As in GRPO, this advantage is a single scalar per completion, shared across all token indices.
iGRPO additionally uses the token-averaged form (the $1/|o_j|$ factor) so that completions of
different lengths contribute comparably:
\begin{equation}
g_{\mathrm{on\text{-}policy}}(\theta \mid q')
\;\doteq\;
\frac{1}{G}\sum_{j=1}^G
\frac{\hat A_j}{|o_j|}
\sum_{t=1}^{|o_j|}
\nabla_\theta \log \pi_\theta(o_{j,t} \mid q', o_{j,<t}).
\label{eq:onpolicy_group_gradient_estimator}
\end{equation}
Eq.~\eqref{eq:onpolicy_group_gradient_estimator} is the basic (conceptual) iGRPO policy-gradient
estimator when sampling from $\pi_\theta$. The remainder of the derivation follows the standard
PPO/GRPO stabilization used in the main methodology: sampling from $\pi_{\theta_{\mathrm{old}}}$
and using importance ratios with clipping and a KL penalty.

\subsection{Off-policy sampling from $\pi_{\theta_{\mathrm{old}}}$ and PPO-style clipping}

Within each iteration, iGRPO samples Stage~2 completions from $\pi_{\theta_{\mathrm{old}}}$ for
stability. To relate gradients under $\pi_{\theta_{\mathrm{old}}}$ to the updated policy
$\pi_\theta$, define the per-token importance ratio
\begin{equation}
r_{j,t}(\theta)
=
\frac{\pi_{\theta}(o_{j,t} \mid q', o_{j,<t})}
     {\pi_{\theta_{\mathrm{old}}}(o_{j,t} \mid q', o_{j,<t})}.
\label{eq:appendix_ratio}
\end{equation}
Following PPO/GRPO, iGRPO maximizes a clipped surrogate that replaces the on-policy factor
by a clipped importance-weighted term:
\begin{equation}
\mathcal{L}^{\mathrm{clip}}_{j,t}(\theta)
=
\min\!\Bigl(
r_{j,t}(\theta)\,\hat A_j,\;
\mathrm{clip}\bigl(r_{j,t}(\theta),\,1-\epsilon,\,1+\epsilon\bigr)\,\hat A_j
\Bigr).
\label{eq:appendix_clipped_surrogate}
\end{equation}
The per-token clipped objective in Eq.~\eqref{eq:appendix_clipped_surrogate} yields the usual
piecewise gradient behavior: when the ratio is clipped, the clipped branch is constant in
$r_{j,t}(\theta)$ and contributes zero gradient through that branch.

A convenient way to write the gradient is via the indicator of the \emph{unclipped} branch.
Let
\begin{equation}
\mathbb{I}_{j,t}(\theta)
=
\begin{cases}
1, & \hat A_j \ge 0 \text{ and } r_{j,t}(\theta) \le 1+\epsilon, \\
1, & \hat A_j < 0 \text{ and } r_{j,t}(\theta) \ge 1-\epsilon, \\
0, & \text{otherwise},
\end{cases}
\label{eq:appendix_indicator}
\end{equation}
which matches the standard PPO clipping rule. Then, using
$\nabla_\theta r_{j,t}(\theta) = r_{j,t}(\theta)\,\nabla_\theta \log \pi_\theta(o_{j,t}\mid q',o_{j,<t})$,
the gradient of the clipped surrogate is
\begin{equation}
\nabla_\theta \mathcal{L}^{\mathrm{clip}}_{j,t}(\theta)
=
\mathbb{I}_{j,t}(\theta)\;
\hat A_j\; r_{j,t}(\theta)\;
\nabla_\theta \log \pi_\theta(o_{j,t} \mid q', o_{j,<t}).
\label{eq:appendix_grad_clipped}
\end{equation}

\subsection{Including the per-token KL penalty}

As in Section~\ref{sec:background-grpo}, iGRPO includes a per-token KL penalty to a reference
policy $\pi_{\mathrm{ref}}$ using the non-negative estimator
\begin{equation}
\widehat{D}_{\mathrm{KL}}^{(j,t)}
=
\frac{\pi_{\mathrm{ref}}(o_{j,t}\mid q',o_{j,<t})}{\pi_{\theta}(o_{j,t}\mid q',o_{j,<t})}
-
\log
\frac{\pi_{\mathrm{ref}}(o_{j,t}\mid q',o_{j,<t})}{\pi_{\theta}(o_{j,t}\mid q',o_{j,<t})}
-1.
\label{eq:appendix_kl_estimator}
\end{equation}
Let $\rho_{j,t}(\theta) = \pi_{\mathrm{ref}}(o_{j,t}\mid q',o_{j,<t}) / \pi_\theta(o_{j,t}\mid q',o_{j,<t})$.
Then $\widehat{D}_{\mathrm{KL}}^{(j,t)} = \rho_{j,t} - \log \rho_{j,t} - 1$, and its gradient has a
simple form:
\begin{equation}
\nabla_\theta \widehat{D}_{\mathrm{KL}}^{(j,t)}
=
-\bigl(\rho_{j,t}(\theta)-1\bigr)\;
\nabla_\theta \log \pi_\theta(o_{j,t}\mid q',o_{j,<t}).
\label{eq:appendix_grad_kl}
\end{equation}
Therefore, the KL-regularized contribution in the objective,
$-\beta\,\widehat{D}_{\mathrm{KL}}^{(j,t)}$, contributes
\begin{equation}
-\beta\,\nabla_\theta \widehat{D}_{\mathrm{KL}}^{(j,t)}
=
\beta\bigl(\rho_{j,t}(\theta)-1\bigr)\;
\nabla_\theta \log \pi_\theta(o_{j,t}\mid q',o_{j,<t}).
\label{eq:appendix_grad_kl_penalty}
\end{equation}

\subsection{Final iGRPO surrogate objective and resulting policy gradient}

Putting the pieces together and reinstating the full two-stage sampling (where Stage~1 affects
the distribution of $q'$ but is not differentiated through within an iteration), the iGRPO
surrogate objective is:
\begin{align}
\mathcal{J}_{\mathrm{iGRPO}}(\theta)
&=
\mathbb{E}\Bigl[q \sim P(Q)\Bigr]\;
\mathbb{E}\Bigl[
\{d_i\}_{i=1}^N \sim \pi_{\theta_{\mathrm{old}}}(\cdot \mid q),\;
\hat d = \arg\max_i R_\phi(d_i),\;
q'=\mathrm{Concat}(q,\hat d), \nonumber\\
&\hspace{7.8em}
\{o_j\}_{j=1}^G \sim \pi_{\theta_{\mathrm{old}}}(\cdot \mid q')
\Bigr] \nonumber\\
&\quad\times
\frac{1}{G}\sum_{j=1}^{G}\frac{1}{|o_j|}\sum_{t=1}^{|o_j|}
\Bigl[
\mathcal{L}^{\mathrm{clip}}_{j,t}(\theta)
-
\beta\,\widehat{D}_{\mathrm{KL}}^{(j,t)}
\Bigr],
\label{eq:appendix_final_objective}
\end{align}
where $\mathcal{L}^{\mathrm{clip}}_{j,t}(\theta)$ is defined in Eq.~\eqref{eq:appendix_clipped_surrogate},
$r_{j,t}(\theta)$ in Eq.~\eqref{eq:appendix_ratio}, and $\widehat{D}_{\mathrm{KL}}^{(j,t)}$ in
Eq.~\eqref{eq:appendix_kl_estimator}. This matches the structure given in Eq.~\eqref{eq:igrpo_objective}.

Differentiating Eq.~\eqref{eq:appendix_final_objective} yields the iGRPO policy gradient:
\begin{align}
\nabla_\theta \mathcal{J}_{\mathrm{iGRPO}}(\theta)
=
\mathbb{E}\Bigl[\cdots\Bigr]\;
\frac{1}{G}\sum_{j=1}^{G}\frac{1}{|o_j|}\sum_{t=1}^{|o_j|}
\Bigl[
\nabla_\theta \mathcal{L}^{\mathrm{clip}}_{j,t}(\theta)
-
\beta\,\nabla_\theta \widehat{D}_{\mathrm{KL}}^{(j,t)}
\Bigr],
\label{eq:appendix_final_gradient}
\end{align}
with the explicit per-token forms from Eqs.~\eqref{eq:appendix_grad_clipped} and
\eqref{eq:appendix_grad_kl}. Concretely, combining them gives
\begin{align}
\nabla_\theta \mathcal{J}_{\mathrm{iGRPO}}(\theta)
=
\mathbb{E}\Bigl[\cdots\Bigr]\;
\frac{1}{G}\sum_{j=1}^{G}\frac{1}{|o_j|}\sum_{t=1}^{|o_j|}
\Bigl[
\mathbb{I}_{j,t}(\theta)\,\hat A_j\,r_{j,t}(\theta)
+
\beta\bigl(\rho_{j,t}(\theta)-1\bigr)
\Bigr]\;
\nabla_\theta \log \pi_\theta(o_{j,t}\mid q',o_{j,<t}).
\label{eq:appendix_final_gradient_explicit}
\end{align}

\paragraph{Interpretation.}
Eq.~\eqref{eq:appendix_final_gradient_explicit} makes the roles of the two stages explicit.
Stage~1 defines the self-conditioned context $q'$ (and thus the learning problem presented to
Stage~2) but does not contribute direct gradients within an iteration. Stage~2 contributes a
standard GRPO/PPO-style token-level policy gradient, where each token is weighted by a
group-normalized advantage $\hat A_j$ (shared across tokens of the completion), stabilized by
importance-ratio clipping, and regularized by a KL penalty to $\pi_{\mathrm{ref}}$.

\begin{figure}[t]
    \centering
    \includegraphics[width=0.65\textwidth]{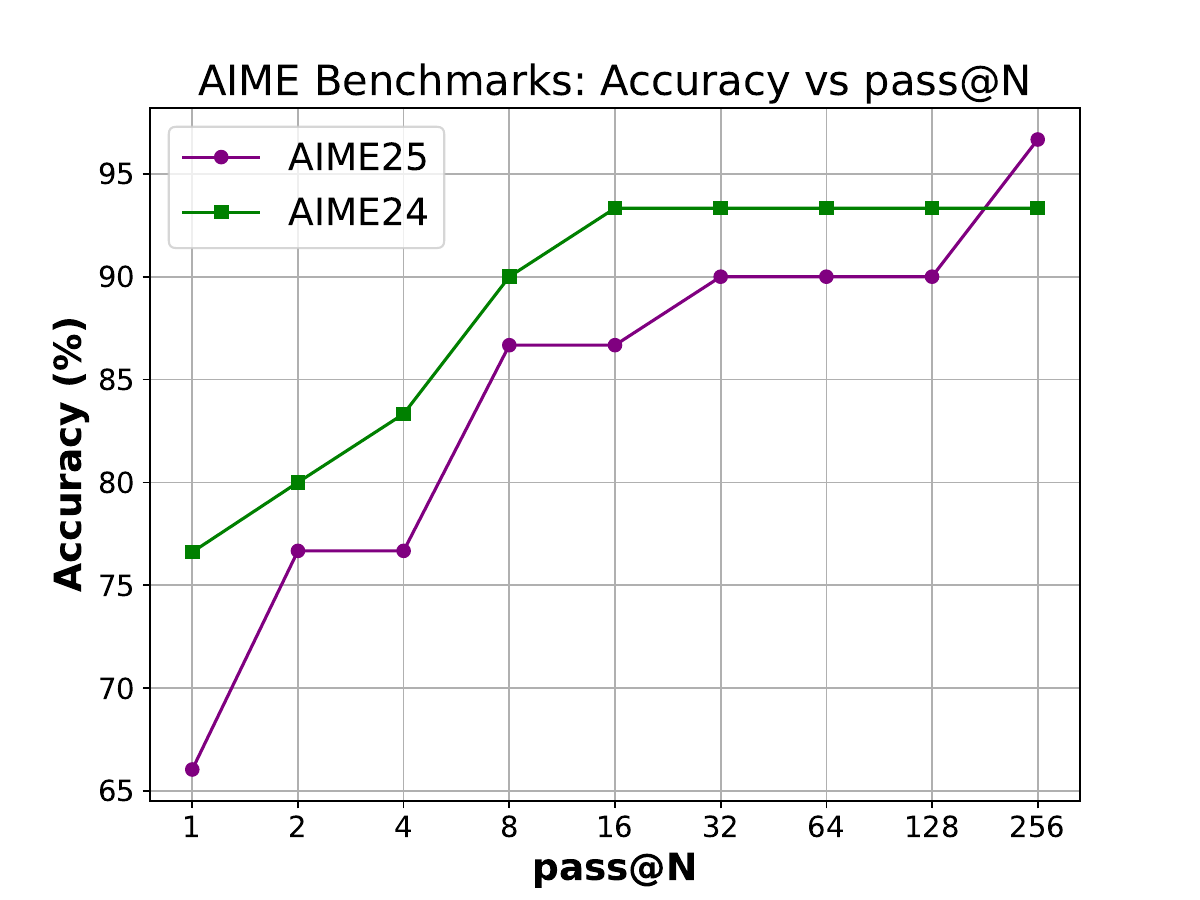}
    \caption{
        Performance of iOpenMath-Nemotron-14B across various pass@N settings for AIME24 and AIME25. 
        Both benchmarks exhibit increasing accuracy with higher \(N\), though AIME24 quickly stabilizes 
        at 93.33\% by \(N=16\), whereas AIME25 continues to rise until reaching 96.67\% at \(N=256\).
    }
    \label{fig:pass_n_aime}
\end{figure}

\begin{table*}[t]
\centering
\caption{Performance comparison of OpenMath-Nemotron-14B and iGRPO-enhanced variant.}
\label{tab:openmath_app}
\resizebox{\textwidth}{!}{%
\begin{tabular}{lccccccc}
\hline
\rowcolor{gray!20}
\textbf{Model} & \textbf{AIME25} & \textbf{AIME24} & \textbf{MATH500} & \textbf{AMC} & \textbf{GSM8K} & \textbf{Minerva} & \textbf{Avg} \\
\hline
OpenMath-Nemotron-14B~\citep{moshkov2025aimo}       & 61.18 & 73.28 & 95.55 & 95.00 & 94.01 & 33.46 & 75.41 \\
\rowcolor{gray!15}
\textbf{OpenMath-Nemotron-14B + iGRPO}        & \textbf{66.04} & \textbf{76.61} & \textbf{96.90} & \textbf{97.50} & \textbf{94.16} & \textbf{38.24} & \textbf{78.24} \\
\hline
\end{tabular}%
}
\end{table*}

\section{Scaling OpenMath-Nemotron-14B with iGRPO}
To further validate the effectiveness of iGRPO, we trained the OpenMath-Nemotron-14B with iGRPO on large scale dataset of OpenR1-Math-220k~\citep{openr1_math_220k} which consists of 220,000 math problems with reasoning traces from DeepSeek-R1. For this study, we use 94,000 examples. As shown in Table~\ref{tab:openmath_app}, the reasoning performance of the model trained with iGRPO is significantly improved, achieving an impressive AIME25 score of 66.04\%. 

\subsection{Analysis of Pass@N on AIME Benchmarks}
Figure~\ref{fig:pass_n_aime} shows how the accuracy of iOpenMath-Nemotron-14B evolves when increasing the number of attempts \(N\) on the AIME24 and AIME25 benchmarks. As expected, we see consistent gains in performance as \(N\) grows, indicating that the model can generate correct solutions among multiple sampled responses even if the top-1 guess is sometimes incorrect. Although our base SFT model already shows strong results at pass@1, the additional RL fine-tuning appears to capitalize on multi-sample scenarios. For example, on AIME25, the model improves from 66.04\% (pass@1) to 86.67\% at pass@8, demonstrating the effectiveness of producing multiple solutions for challenging competition problems.

Despite these gains, there are clear saturation points. AIME24 converges to its best score of 93.33\% by \(N=16\), with no further improvement at higher \(N\). On the contrary, the performance of AIME25 continues to improve even at high values of \(N\). While the performance seems to plateau briefly at 90.00\% between \(N=32\) and \(N=128\), it eventually increases to 96.67\% at \(N=256\). This contrast suggests that certain problem distributions, especially in AIME25, may benefit from a larger number of sampled attempts, whereas others, such as those in AIME24, can be adequately solved with fewer solution attempts.

\begin{table}[t]
\centering
\caption{Hyperparameters for training 7B models with iGRPO using 2 nodes each using 8 × NVIDIA A100 GPUs. One node is entirely reserved for vllm.}
\begin{tabular}{ll}
\toprule
\textbf{Parameter} & \textbf{Value} \\
\midrule
bf16 & true \\
attn\_implementation & flash\_attention\_2 \\
use\_vllm & true \\
vllm\_gpu\_memory\_utilization & 0.85 \\
gradient\_accumulation\_steps & 8 \\
gradient\_checkpointing & true \\
learning\_rate & 1e-06 \\
lr\_scheduler\_type & cosine\_with\_min\_lr \\
lr\_scheduler\_kwargs & min\_lr\_rate: 0.1 \\
warmup\_ratio & 0.1 \\
num\_train\_epochs & 1 \\
per\_device\_train\_batch\_size & 16 \\
max\_completion\_length & 4096 \\
num\_generations & 4 \\
temperature & 0.7 \\
reward\_funcs & accuracy, format \\
reward\_weights & 1.0, 1.0 \\
\bottomrule
\end{tabular}
\label{tab:hypers_7b}
\end{table}

\section{Hyperparameter Setup}
\label{sec:hyperparams}

We conduct ablation studies on both 7B and 14B model variants, spanning core architectures such as DeepSeek-R1-Distill-Qwen and OpenMath-Nemotron. All models are trained for one epoch across two datasets: (1) the MATH dataset~\citep{hendrycks2021measuring} of 7,500 step-by-step problems, and (2) AceReason‑Math~\citep{chen2025acereason} dataset (9400 problems). Our training uses a KL divergence loss coefficient of 0. We set a cosine learning rate schedule (minimum rate of 0.1) with a base learning rate of \(1\times10^{-6}\). We use 8 rollouts for all experiments. 

Table~\ref{tab:hypers_7b} lists the concrete hyperparameters for training 7B iGRPO models. Notably, we run on 2 nodes with 8 $\times$ NVIDIA A100 GPUs each, and one of these nodes is fully allocated to vLLM for generation. We keep a global batch size of 128, with a per-device batch size of 16 and a gradient accumulation step size of 8. For the 14B models, we scale out to 5 nodes of 8 $\times$ NVIDIA A100 GPUs each (again, one node reserved for vLLM) and reduce the per-device batch size to 4 (maintaining the same global batch size of 128). In both 7B and 14B setups, we continue to use bfloat16 precision and the FlashAttention-2 kernel. The temperature is set to 0.7 for generation, and we apply two reward functions (accuracy and format) each with weight~1.0. This configuration provides a balanced trade-off between training stability, throughput, and alignment with complex mathematical reasoning tasks.

\paragraph{Prompt:} We use the following prompt for training model with iGRPO.

\begin{takeaway}{teal}{Prompt}
{You are a helpful AI Assistant that provides well-reasoned and detailed responses. First think through your reasoning as an internal monologue, then give the answer: \textless think\textgreater …\textless /think\textgreater\textless answer\textgreater …\textless /answer\textgreater. If the prompt contains feedback or a prior draft, treat it as guidance, not as ground truth. Use it to produce a strictly improved final answer: fix mistakes, fill gaps, strengthen justification, and improve clarity. Do not repeat the draft verbatim. If the feedback is wrong or incomplete, correct it and proceed.}
\end{takeaway}

\section{Memory and Throughput Comparisons}
\subsection{Setup}

To evaluate the resource utilization of our training setup, we replicate the exact environment and conditions under which our 7B models are typically trained. Specifically, we use DeepSeek-R1-Distill-Qwen-7B as our base model, which serves as a representative checkpoint for measuring throughput and memory consumption when training on the MATH dataset. Our training configuration employs a per-device batch size of 16, along with a global gradient accumulation step of 8, allowing us to effectively simulate heavier loads without exceeding GPU memory constraints. Additionally, we use a maximum completion length of 2048 tokens to benchmark model performance. We run experiments on two nodes, each equipped with 8 × NVIDIA A100 GPUs. One node is dedicated to vllm generation, ensuring that inference or generation processes do not interfere with the primary training workload, while the other node is reserved exclusively for model training.

We measure peak memory usage by periodically querying the GPU memory allocator for the maximum memory allocation that has occurred since the start of training. Specifically, at the beginning of training, we reset the peak memory statistics, and then after each iteration, we retrieve the current peak memory usage in bytes. We convert this value to gigabytes for readability and log it alongside other training metrics. To measure throughput, we track the total number of samples processed over time. We calculate this by multiplying the current global step by both the per-device batch size and the number of devices used in data parallelism. Dividing this product by the elapsed training time in seconds yields the throughput, expressed as samples processed per second. This real-time monitoring of memory and throughput allows us to evaluate hardware utilization efficiency, identify possible bottlenecks, and compare different training configurations in a consistent and quantifiable manner.

\subsection{Measurements}

Table~\ref{tab:mem_throughput} presents measured GPU usage and training throughput under iGRPO vs GRPO. Despite the two-stage nature of iGRPO, its peak memory usage of 54.9349\,GB closely matches GRPO’s 54.9286\,GB, a difference of roughly 0.0063\,GB, which is practically negligible. This matches our theoretical expectation that the self-feedback mechanism adds minimal overhead, validating the feasibility of integrating iterative refinements even under constrained resource budgets.

Regarding throughput, iGRPO processes 0.34\,samples/s compared to GRPO’s 0.41\,samples/s, reflecting a mild slowdown tied to the additional round of generation. Crucially, this is neither an order-of-magnitude nor a large factor reduction. Instead, it shows that iGRPO’s second-stage refinement imposes only a modest computational cost. In summary, these measurements confirm our claims that iGRPO can be implemented with little additional overhead, supporting it as a practical strategy for enhancing mathematical reasoning performance without compromising resource efficiency.

Beyond instantaneous throughput, we also report full training cost measured in total GPU hours. Under the same compute budget and eight generations per prompt, GRPO requires 83.3 GPU hours while iGRPO uses 94.1 GPU hours, which corresponds to roughly a 13\% increase in wall-clock training time. This overhead arises from the sequential Stage~1 plus Stage~2 decoding but does not demand more GPUs or additional memory capacity, since peak usage remains essentially unchanged. Given that this modest time increase delivers several-point gains on AIME24 and AIME25 and enables our 7B models to reach state-of-the-art performance, we view the tradeoff between 13\% extra training time and substantially higher reasoning accuracy as a favorable and practical value proposition in real deployments.

\begin{table}[h]
\centering
\resizebox{0.57\linewidth}{!}{%
\begin{tabular}{lccc}
\toprule
\textbf{Method} & \textbf{Peak Memory (GB)} & \textbf{Throughput (Samples/sec)} & \textbf{Total GPU Hours} \\
\midrule
GRPO  & 54.9286 & 0.41 & 83.3 \\
iGRPO & 54.9349 & 0.34 & 94.1 \\
\bottomrule
\end{tabular}
}
\caption{Empirical memory usage and throughput on an 80\,GB A100 setup. iGRPO’s two-stage approach yields near-identical peak memory usage and only a minor decrease in throughput compared to GRPO. The last column reports total GPU hours for a full training run, showing that iGRPO adds only a modest 13\% time overhead relative to GRPO.}
\label{tab:mem_throughput}
\end{table}

\section{Additional Ablation Studies}

\paragraph{Training Dynamics and Response Length.}
As shown in Fig.~\ref{fig:reward_length}, we compare average rewards for iGRPO and GRPO at multiple checkpoints, observing that iGRPO consistently maintains a higher reward throughout training. The iterative refinement in iGRPO ultimately yields a superior reward trajectory. In addition, we measure the response length over training steps and find that both methods exhibit nearly identical lengths, with GRPO producing slightly longer outputs on average. Notably, iGRPO’s two-stage process does not manifest in lengthy completions but instead appears to refine solutions within a similar token budget. This indicates that the gains from iterative refinement arise more from improved response quality than from verbosity.

\begin{figure}[t]
  \centering
  \begin{subfigure}[b]{0.48\linewidth}
    \centering
    \includegraphics[width=\linewidth]{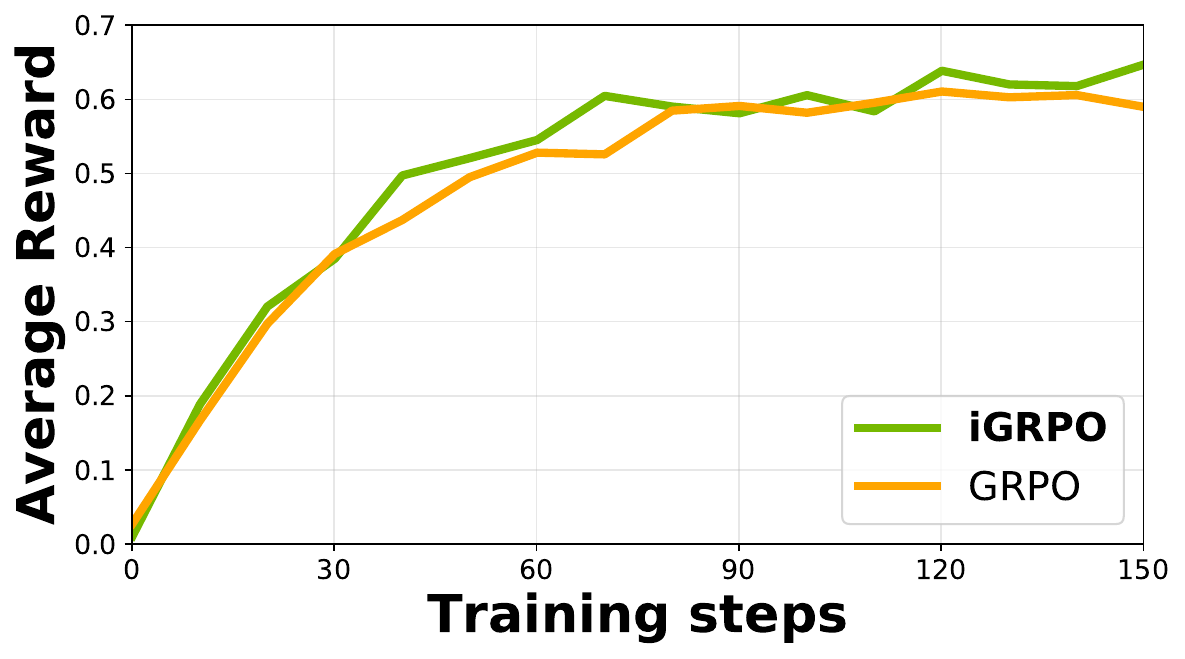}
    \caption{Average training reward.}
    \label{fig:reward}
  \end{subfigure}\hfill
  \begin{subfigure}[b]{0.48\linewidth}
    \centering
    \includegraphics[width=\linewidth]{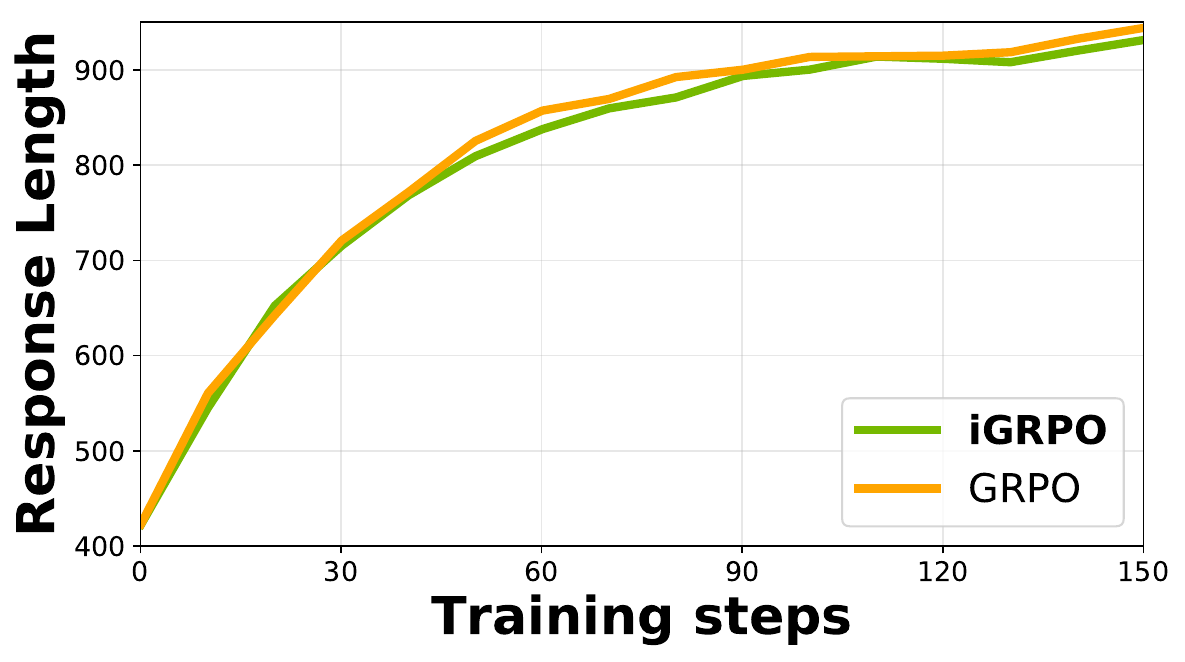}
    \caption{Response length.}
    \label{fig:length}
  \end{subfigure}
  \caption{Comparison of (a) average training rewards and (b) response lengths for GRPO vs.\ iGRPO.}
  \label{fig:reward_length}
\end{figure}

\paragraph{Effect of KL Divergence Term.}
We vary the coefficient $\beta \in \{0, 0.0001, 0.001, 0.01\}$ to examine how tightly the policy is regularized against the reference model. As shown in Table~\ref{tab:ablation}, while $\beta = 0.0001$ achieves the highest overall score (70.23\%), the difference among all settings is relatively small. The KL term, in principle, balances exploration with adherence to the current policy. However, given the marginal gains observed, setting $\beta = 0$ offers a simpler training pipeline without sacrificing significant performance. Hence, we use $\beta = 0$ to reduce overhead and maintain efficiency.

\begin{table}[t]
\centering
\caption{Ablation results for iGRPO. We study the effect of the KL divergence coefficient $\beta$ (top) and the total number of completions (bottom).}
\label{tab:ablation}
\begin{tabular}{c c}
\toprule
\multicolumn{2}{c}{\textbf{Effect of KL Divergence Coefficient $\beta$}} \\
\midrule
$\beta$ & Score (\%) \\
\midrule
0       & 69.87 \\
0.0001  & \textbf{70.23} \\
0.001   & 69.31 \\
0.01    & 69.91 \\
\midrule
\multicolumn{2}{c}{\textbf{Effect of Number of Completions}} \\
\midrule
Total Completions & Score (\%) \\
\midrule
4   & 67.79 \\
8   & 69.87 \\
16  & 70.17 \\
32  & \textbf{70.33} \\
\bottomrule
\end{tabular}
\end{table}

\paragraph{Effect of Number of Completions.}
We study how the total number of completions in iGRPO affects performance, allocating $4, 8, 16,$ or $32$ completions evenly across the two stages. As shown in Table~\ref{tab:ablation}, increasing from $4$ to $8$ completions gives a clear improvement, while gains beyond $8$ are modest. Larger budgets also increase training time and inference latency for minimal returns.

\end{document}